%% file: thesis_archive.tex
\documentclass[11pt]{article}
\usepackage{amsfonts}
\usepackage{hyperref}
\usepackage{graphicx}
\usepackage{subfigure}
\usepackage{epsfig}
\usepackage{amsmath}
\usepackage{amssymb}
\usepackage{enumerate}
\usepackage{mathabx}
\usepackage{mathrsfs}
\usepackage{tabularx,ragged2e,booktabs,caption}
\usepackage{enumitem}
\usepackage{bbm}
\usepackage{hujititlepage}
\usepackage{pdfpages}
\usepackage[noend]{algorithmic}
\usepackage{natbib}

\newcommand{\DNF}{\mathrm{DNF}}

\newcommand{\DT}{\mathrm{DT}}

\newcommand{\x}{\mathbf{x}}
\newcommand{\z}{\mathbf{z}}

\input{header.tex}

\begin{document}


\title{Learning Using Local Membership Queries}
\author{Galit Bary}
\date{September 2015}
\maketitle


\begin{center}
\section*{\large{Abstract}}
\end{center}

Classic machine learning algorithms learn from labelled examples. For example, to design a machine translation system, a typical training set will consist of English sentences and their translation to French. There is a stronger model, in which the algorithm can also query for labels of new examples it creates. E.g, in the translation task, the algorithm can create a new English sentence, and request its translation from the user during training. This combination of examples and queries, that resembles human learning patterns, has been widely studied. Yet, despite many theoretical results, query algorithms are almost never used.
One of the main causes for this is a report~\citep{baum1992query} on very disappointing empirical performance of a query algorithm. These poor results were mainly attributed to the fact that the algorithm queried for labels of examples that are artificial, and impossible to interpret by humans. 

In this work we study a new model of \textit{local} membership queries  \citep{awasthi2012learning}, which tries to resolve the problem of artificial queries. In this model, the algorithm is only allowed to query the labels of examples which are close to examples from the training set. E.g., in translation, the algorithm can change individual words in a sentence it has already seen, and then ask for the translation. In this model, the examples queried by the algorithm will be close to natural examples and hence, hopefully, will not appear as artificial or random. In this work we focus on 1-local membership queries (i.e., queries of distance 1 from an example in the training sample). We show that 1-local membership queries are already stronger than the standard learning model.
We also present an experiment on a well known NLP task of sentiment analysis. In this experiment, the users were asked to provide, in a way that resembles 1-local queries, more information than merely indicating the label. We present results that illustrate that this extra information is beneficial in practice. 

\newpage
\begin{center}
\section*{Acknowledgments}
\end{center}
I would like to thank my advisor Prof. Shai Shalev-Shwartz for having me on his outstanding team and for his support and inspiration. I would also like to thank Amit Daniely, for his guidance and mentorship. His extensive knowledge and patience were invaluable. It has been a privilege to work with him.
 
To Alon Gonen, Nir Rosenfled, Yoav Wald, Yossi Arjevani, Nomi Vinokurov and Avishai Wagner for their remarkable friendship and counsel. To the NLP lab and especially Effi Levi for all his assistance in the empirical work and to my officemates Zahi Ajami and Dikla Cohn for their wonderful companionship.

I would like to thank my parents for all the love and support throughout the years, and to the Weisberg family for their help, especially Susan for her editorial comments. Last but not least, I would like to thank my husband Dov for his enduring support that is expressed on so many levels -- encouraging me during difficult times, editing my drafts and providing home cooked meals.
\newpage

\tableofcontents
\newpage

\section{Introduction} \label{intro}

\textit{How do humans learn}? 
Say we look at the process of a child learning how to recognize a cat. We can focus on two types of input. The first type of input is when a child's parent points at a cat and states ``Look, a cat!". The second type of input is an answer to the child's frequent question ``What is that?", which the child may pose when seeing a cat, but also when seeing a dog, a mouse, a rabbit, or any other small animal.

These two types of input were the basis for the learning model originally suggested in the celebrated paper ``A theory of the learnable" \citep{valiant1984theory}. In Valiant's learning model, the learning algorithm has access to two sources of information - EXAMPLES and ORACLE. The learning algorithm can call EXAMPLES to receive an example with its label (sampled from the ``nature"). Additionally, the learning algorithm can use ORACLE, which provides the label of {\em any} example presented to it.
With these two input types, we can look at two models of learning: learning using only calls for EXAMPLES, and learning using calls for both EXAMPLES and ORACLE. The first is the standard Probably Approximately Correct (PAC) model. The second is the so called PAC+MQ (Membership Queries) model. There has been a lot of theoretical work searching for the limits of the additional strength of membership queries. The use of membership queries in addition to examples was proven to be stronger than the standard PAC model in many cases \citep{angluin1987learning, blum1992fast, bshouty1995exact,jackson1994efficient}(see section  \ref{prev}). 

Despite that the MQ model seems much stronger, both intuitively and formally, it is rarely used in practice. 
This is commonly believed to result from the fact that in many cases it is not easy to implement MQ algorithms, that can create new and artificial examples to be labeled as part of the training phase. This problem of labeling artificial examples was highlighted by the experiment of  \cite{baum1992query}. Baum and Lang implemented a membership query algorithm proposed by \cite{baum1991neural} for learning halfspaces . Their algorithm had very poor results, which was attributed to the fact that the algorithm created artificial and unnatural examples, which resulted in a noisy labeling. We elaborate on this experiment and criticize its conclusions in section \ref{prev}.

A suggested solution to the problem of unnatural examples was proposed by \cite{awasthi2012learning}. They suggested a mid-way model of learning with queries, but only restricted ones. The queries that their model allows the algorithm to ask are only \textit{local} queries, i.e., queries that are close in some sense to examples from the sample set. 
Hopefully, examples which are similar to natural examples will also appear to be natural, or at least close to natural, and in any case will be far from appearing random or artificial.
In their work, Awasti et al.  started to investigate the power and the limitations of this model of local queries. They proved positive results on learning sparse polynomials with $O(\log(n))$-local queries under what they defined as \textit{locally smooth distributions}\footnote{locally $\alpha$-smooth distributions can be defined as the class of distributions for which the logarithm of the density function is $\log(\alpha)$-Lipschitz with respect to the Hamming distance.}, which in some sense generalize the uniform and product distributions.
They also proposed an algorithm that learns DNF formulas under the uniform distribution in quasi-polynomial time using only $O(\log(n))$-local queries.


The exciting ideas of \cite{awasthi2012learning} leave many directions for future work. One issue is that their analysis holds for a restricted family of distributions. While these results provide evidence of the excessive power of local queries, the distributional assumptions are rather strong.

Our work follows Awasthi et al., and is focused on 1-local queries, which are the closest to the original PAC model.
We formulate an arguably natural distributional assumption, and present an algorithm that uses 1-local membership queries to learn DNF formulas under this assumption. 
We also provide a matching lower bound: Namely, we prove that learning DNFs under our assumption is hard without the use of queries, assuming that learning decision trees is hard. This is the first example of a natural problem in which 1-local queries are stronger than the vanilla PAC model (it complements the work of Awasthi et al. who showed a similar result for a highly artificial problem).

Finally, we provide some empirical evidence that using local queries can be helpful in practice, and importantly, that the implementation of the queries is easy, straightforward, and can be acquired by crowdsourcing without the use of an expert. We present a method for using local queries to perform a user-induced feature selection process, and present results of this protocol on the task of sentiment analysis of tweets. 
Our results show that by acquiring a more expressive data set, using (a variant of) 1-local queries, we can achieve better results with fewer examples. Based on the fact that a smaller data set is sufficient, we gain twice: we need less manpower for the labeling process
and less computing power for the training process.
We note that similar experiments also present encouraging results along this line \citep{raghavan2007interactive, raghavan2005interactive, settles2011closing, druck2009active}. This supplies more evidence that such query-based methods can be useful in practice.
 \pagebreak

\section{Previous Work} \label{prev}
\subsection{PAC}

Valiant's Probably Approximately Correct (PAC) model of learning \citep{valiant1984theory} formulates the problem of learning a concept from examples. Examples are chosen according to a fixed but unknown and arbitrary distribution on the instance space. The learner's task is to find a prediction rule. The requirement is that with high probability, the prediction rule will be correct on all but a small fraction
of the instances.

A few positive results are known in this model - i.e., concept classes that have been proven to be PAC-learnable. Maybe the most significant example is the class of halfspaces. More examples include relatively weak classes such as DNFs and CNFs with constantly many terms~\citep{valiant1984theory}, and rank $k$ decision trees \citep{ehrenfeucht1989learning} for a constant $k$.

Despite these positive results, most PAC learning problems are probably intractable. In fact, beyond the results mentioned above, almost no positive results are known. Furthermore, several negative results are known. 
For example, learning automatons, logarithmic depth circuits, and intersections of polynomially many halfspaces are all intractable, assuming the security of various cryptographic schemes \citep{kearns1994cryptographic, klivans2006cryptographic}.
In \citep{daniely2014average,daniely2014complexity,daniely2013more}, it is shown that learning DNF formulas, and learning intersections of $\omega(\log(n))$ halfspaces are intractable under the assumption that refuting random $k$-SAT is hard.

\subsection{Membership Queries}

The PAC model is a ``passive" model in which the learner receives a random data set of examples and their labels  and then
outputs a classifier. A stronger version would be an active model in which the learner gathers information about the world by asking questions and receiving responses. Several types of active models have been proposed: the Membership Query Synthesis, Stream-Based Selective Sampling, and Pool-Based Sampling \citep{settles2010active}. 
Our work is in the area of the ``Membership Queries" (MQ) model which was presented in \citep{valiant1984theory}. In this model the learner is allowed to query for the label of any particular example that it chooses (even examples that are not in the given sample).

This model has been shown to be stronger in several scenarios.
Some examples of concept classes that have been proven to be PAC-learnable only if membership queries are available include:
The class of Deterministic Finite Automatons \citep{angluin1987learning}, the class of k-term DNF for $k= \frac{\log(n)}{\log(\log(n))}$ \citep{blum1992fast}, the class of decision trees and k-almost monotone-DNF formulas \citep{bshouty1995exact}, the class of intersections of k-halfspaces~\citep{baum1991neural} and the class of DNF formulas under the uniform distribution \citep{jackson1994efficient}. The last of these results was built upon 
Freund's boosting algorithm \citep{freund1995boosting} and the Fourier-based technique for learning using membership queries due to \citep{kushilevitz1993learning}.
 
It should be noted that there are cases in which the additional strength of MQ does not help. E.g., in the case of learning DNF and CNF formulas \citep{angluin1995won}, and in the case of distribution free agnostic learning (although in the distribution-specific agnostic setting membership
queries do increase the power of the learner) \citep{feldman2009power}.

\subsection{Baum and Lang}
As discussed above, there has been widespread and significant theoretical work in the PAC + MQ model. On the other hand, almost no practical work on implementing these ideas has been done.
A well-known exception is the work of \cite{baum1992query}. 
They applied a variation of the MQ algorithm for learning a linear classifier proposed in \cite{baum1991neural}. This algorithm uses the idea that given two examples, one positive and one negative, and a query oracle, it is possible to find an approximately accurate separating halfspace by using a binary search on the line between the positive and negative examples. Their experiment attempts to evaluate this idea in practice. The task that they chose is the task of binary digit classification. The algorithm would receive two examples, one positive and one negative (say, an image of the digit 4 and an image of the digit 7) and would return the weights of the halfspace. The generalization error of the halfspace would then be tested on other examples from the data.  The query technique they used in the experiment is different than in the original algorithm: ``A direct implementation of this algorithm would repeatedly flash images on the screen during the binary search and would require the test subject to type in the correct label for each image. Because this process seemed likely to be error prone, we instead provided an interface that permitted the test subject to scan through the input space using the mouse and then click on an image that seemed to lie right at the edge of recognizability" (from \cite{baum1992query}).

For an example of what the users saw on the screen see figure \ref{fig:L_B_digits}.
\begin{figure}[ht]
\centering
\includegraphics[width=.4\textwidth]{{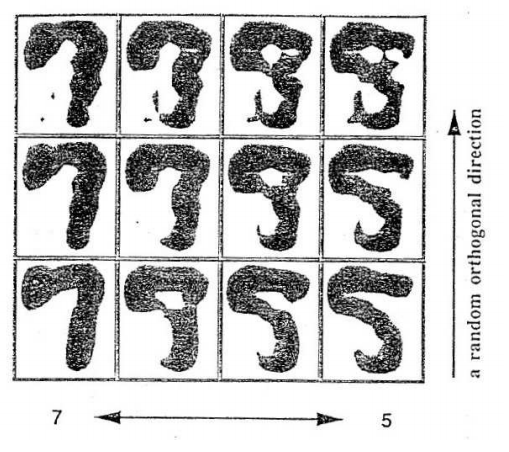}}
\caption{An example taken from \citep{baum1992query}: the images the user saw on the screen for the digits 5 and 7}
\label{fig:L_B_digits}
\end{figure}

They compared the performance of their algorithm to five other variants, three classic PAC (sample based) algorithms:  Backpropogation, Perceptron and simplex, and two baselines: the first returns the perpendicular bisection of the line segments connecting the two examples, and the second returns a randomly oriented hyperplane through the midpoint of the line. The query learning algorithm uses the additional information obtained from the users as described above, while the three PAC algorithms use additional examples drawn from the data set. 
All three PAC algorithms outperformed the query-based algorithm. More surprisingly, even the baseline of choosing the perpendicular bisection line had significantly better results than the halfspace created by the query algorithm. The only method that was worse than the query based method was the random bisector method. They suggest that the reason for the poor results is that the question the users had to answer, to find the boundary pattern, lay outside the range of the human competence. 

This work led many to the conclusion that membership queries  are not useful in practice (\cite{settles2010active,balcan2006agnostic,dasgupta2004analysis} and more). We argue that there are several problems with this conclusion. First and foremost, the task that the users were asked to perform (scanning through images and finding the boundary between digits) is not an intuitive task, and it is very easy to think of other variants for queries which would be more suitable. It is therefore not surprising that the labeling turned out to be noisy considering the nature of the question at hand. Second, their algorithm did not use the PAC abilities; it used queries but did not use the additional option to sample extra points for the data.

\subsection{Local Membership Queries }

Several suggestions have been made of ways to solve the problem of the algorithm's generation of unnatural examples. The most common one was to drop the whole framework of membership queries and focus on the other types of active learning: stream-based and pool-based. The idea is to filter existing examples taken from a large unlabeled data set drawn from the distribution rather than creating artificial examples.
Another suggestion is to give the human annotator the option of answering ``I don't know", or to be tolerant of some incorrect answers. The theoretical framework is the model of an \textit{incomplete membership oracle} in which the answers to a random subset of the queries may be missing. This notion was first presented in \cite{angluin1994randomly}, and then followed by the notion of \textit{limited} MQ and \textit{malicious} MQ. (\cite{angluin1997malicious,blum1995learning,sloan1994learning,bisht2008learning}).

The third method is to restrict the examples that the learning algorithm can query to examples that are similar to examples drawn from the distribution. This is formalized in the work of \cite{awasthi2012learning}. They present the concept of learning using only \textit{local} membership queries. This framework deals with the problem raised by \citep{baum1992query}. By questioning about examples which are close to examples from the distribution we escape the problem of generating random or non-classifiable examples.

The work of Awasthi et al. focused on the n-dimensional boolean hyper-cube $\cX = \{-1,1\}^n$ and on $O(\log(n))$-local queries, i.e., the learning algorithm is given the option to query the label of any point for which there exists a point in the training sample with hamming distance lower than $O(\log(n))$.
The model they suggested is a mid-way model between the PAC model (0-local queries) and the PAC + MQ model (n-local queries). Their main result is that t-sparse polynomials are learnable under \textit{locally smooth} distributions  using $O\left(\log(n)+\log(t) \right)$-local queries. Another interesting result that they presented is that the class of DNF formulas is learnable under the uniform distribution in quasi-polynomial time ($n^{O(\log\log n})$) using $O(\log(n))$-local queries. They also presented some results regarding the strength of local MQ. They proved that under standard cryptographic assumptions, using $(r+1)$-local queries is more powerful than using $r$-local queries (for every $1\leq r\leq n-1$). They also showed that local queries do not always help. They showed that if a concept class is agnostically learnable under the uniform distribution using $k$-local queries (for constant $k$) then it is also agnostically learnable (under the uniform distribution) in the PAC model.

\subsection{Other Related Work} 
In section~\ref{empirical}, we give some experimental evidence that the use of extra information from the user is helpful.
There have been other works along the same line.
\citet{druck2009active} propose a pool-based active learning approach in
which the user provides “labels” for input features,
rather than instances. The users are asked to provide a ``label" for input features, where a labeled input feature denotes that a particular feature is highly indicative of a particular label. 
Following that, \cite{settles2011closing} presented an active
learning annotation interface, in which the users label instances and features simultaneously. At any point in time, an instance and a list of features for each label is presented on the screen. The user can choose to either label the instance, choose a feature from the list as being indicative, or add a new feature of his or her choice. 
Another similar work is of \cite{raghavan2007interactive} and \cite{raghavan2005interactive}. They studied the problem of
tandem learning where they combine uncertainty sampling for instances along
with co-occurrence-based interactive feature selection.
All the above experiments were conducted on the text domain and the features were always unigrams. 
The experiments presented encouraging results of using the human annotators, either by reaching better results, or by showing that the excessive use of annotators can reduce the size of the data set, and sometimes both.

\newpage
\section{Setting} 
\subsection{The PAC Model}
Our framework is an extension of the PAC (Probably Approximately Correct) model of learning. Before introducing it, we will briefly review PAC learning. We will only consider binary classification where the instance space is $\cX=\cX_n=\{-1, 1\}^n$ and the label space is $\cY=\{0,1\}$.
A learning problem is defined by a hypothesis class $\cH\subset \{0,1\}^{\cX}$. 
We assume that the learner receives a \emph{training set}
\[
S = \{(\x_1,h^\star(\x_1)) , (\x_2,h^\star(\x_2)) , \ldots , (\x_m,h^\star(\x_m)) \} \in (\cX \times \cY)^m
\]
where the $\x_i$'s are sampled i.i.d. from some {\em unknown} distribution $\cD$ on $\cX$ and $h^\star:\cX\to\cY$ is some {\em unknown} hypothesis. We will focus on the so-called realizable case where $h^\star$ is assumed to be in $\cH$. The learner returns (a description of) a hypothesis $\hat{h} : \cX \rightarrow \cY$.
The goal is to approximate $h^\star$, namely to find $\hat{h}:\cX\to \cY$ with {\em loss} as small as possible, where the loss is defined as $L_{\cD,h^\star}(\hat{h})=\prob_{\x\sim\cD}\left(\hat{h}(\x)\ne h^\star(\x)\right)$. 
We will require our algorithms to return a hypothesis with loss $< \epsilon$ in time that is polynomial in $n$ and $\frac{1}{\epsilon}$. Concretely,

\begin{definition}[Learning algorithm] \label{learning alg}
We say that a learning algorithm $\cA$ {\bf PAC learns} $\cH$ if
\begin{itemize}
\item
There exists a function $m_\cA \left(n,\epsilon\right)\le \poly\left(n,\frac{1}{\epsilon}\right)$, such that 
for every distribution $\cD$ over $\cX$, every $h^{\star}\in\cH$ and every $\epsilon>0$, if $\cA$ is given a training sequence
\[
S = \{(\x_1,h^\star(\x_1)) , (\x_2,h^\star(\x_2)) , \ldots , (\x_m,h^\star(\x_m)) \} 
\]
where the $\x_i$'s are sampled i.i.d. from $\cD$ and $m \ge m_\cA (n,\epsilon)$, then with probability of at least $\frac{3}{4}$ (over the choice of $S$)\footnote{The success probability can be amplified to $1-\delta$ by repetition.}, the output $\hat{h}$ of $\cA$ satisfies 
$L_{\cD, h^\star}(\hat{h}) < \epsilon$.
\item Given a training set of size $m$
\begin{itemize}
\item $\cA$ runs in time $\poly(m,n)$. 
\item The hypothesis returned by $\cA$ can be evaluated in time $\poly(m,n)$.
\end{itemize}
\end{itemize}

\end{definition}

\begin{definition}[PAC learnability]
We say that a hypothesis class $\cH$ is \textbf{PAC learnable} if there exists a PAC learning algorithm for this class.
\end{definition} 

\subsection{(Local) Membership Queries Model}
Learning with membership queries is an extension of the PAC model in which the learning algorithm is allowed to  \emph{query} the labels of specific examples in the domain set. 
A membership query is a call to an ORACLE which receives as input some $\x \in \cX$ and returns $h^\star(\x)$. This is called a ``membership query" because the ORACLE returns $1$ if $\x$ is in the set of examples positively labeled by $h^\star$.

\begin{definition} [Membership-Query Learning Algorithm] \label{MQ-alg}
We say that a learning \\ algorithm $\cA$ {\bf learns $\cH$ with membership queries} if
\begin{itemize}
\item
There exists a function $m_\cA \left(n,\epsilon\right)\le \poly\left(n,\frac{1}{\epsilon}\right)$, such that 
for every distribution $\cD$ over $\cX$, every $h^{\star}\in\cH$ and every $\epsilon>0$, if $\cA$ is given access to \emph{membership queries}, and a training sequence
\[
S = \{(\x_1,h^\star(\x_1)) , (\x_2,h^\star(\x_2)) , \ldots , (\x_m,h^\star(\x_m)) \} 
\]
where the $\x_i$'s are sampled i.i.d. from $\cD$ and $m \ge m_\cA (n,\epsilon)$, then with probability of at least $\frac{3}{4}$ (over the choice of $S$), the output $\hat{h}$ of $\cA$ satisfies 
$L_{\cD , h^\star}(\hat{h}) < \epsilon$.
\item Given a training set of size $m$
\begin{itemize}
\item $\cA$ asks at most $\poly(m,n)$ membership queries.
\item $\cA$ runs in time $\poly(m,n)$.
\item The hypothesis returned by $\cA$ can be evaluated in time $\poly(m,n)$.
\end{itemize}
\end{itemize}
\end{definition} 
Our work will deal with a specific type of membership queries, ones that are in some way close to examples that are already in the sample. Concretely, we say that a membership query $\x\in\cX$ is {\bf $q$-local} if there exists a training example $x'$ whose Hamming distance\footnote{We only consider the instance space $\{-1,1\}^n$, so the hamming distance is natural. However, the definition can be extended to other metrics.}  from $\x$ is at most $q$. 

\begin{definition} [Local-Query Learning Algorithm] \label{L-MQ-alg}
We say that a learning algorithm $\cA$ {\bf learns $\cH$ with $q$-local membership queries} if $\cA$ learns $\cH$ with membership queries that are all $q$-local.
\end{definition} 

\begin{definition}
We say that a hypothesis class $\cH$ is \textbf{q-LQ learnable} if there exists a q-Local-query learning algorithm for this class.
\end{definition}

\subsubsection*{Learning Under a Specific Family of Distributions}
In the classic PAC model discussed above, the learning algorithm needs to be probably-approximately correct for \emph{any distribution} $\cD$ on $\cX$ and \emph{any hypothesis} $h^\star \in \cH$. In this work we will have guarantees with respect to more restricted families. 
We will say that {\bf $\cA$ learns $\cH$ w.r.t  a family $\mathscr{D}$} of pairs $(\cD , h)$ of distributions on $\cX$ and hypotheses in $ \cH$ if the following holds: The algorithm $\cA$ satisfies the requirements of a learning algorithm whenever the pair $\cD$ and $h$ in the definition of a learning algorithm belongs to $\mathscr{D}$. Similar considerations apply also to the notion of learning with (local) membership queries.

\section{Learning DNFs with Evident Examples Using 1-local MQ}
\subsection{Definitions and Notations}

\label{sec:def}
\begin{definition} [Disjunction Normal Form Formula]
\label{DNF formula}
A {\bf{DNF term}} is a conjunction of literals.
A {\bf{DNF formula}} is a disjunction of DNF terms.  
\end{definition} 
Each DNF formula over n variables naturally induces a function $h : \{-1,1\}^n \to \{0,1\}$ (when we standardly identify $\{0,1\}$ with ``True" and ``False"). We denote by $h_F$ the function induced by the DNF formula $F$.

\begin{remark}
We will look at succinctly described hypotheses (e.g., a $\DNF$ with a small number of terms) and on small, but non-negligible probabilities. For simplicity, we will take the convention that {\bf small} is at most $n^2$ and {\bf non negligible} is at least $\frac{1}{n^3}$. All of our results can be easily generalized to the case where ``small" and ``non-negligible" are defined as $\le n^{c_1}$ and $\ge \frac{1}{n^{c_2}}$ for any constants $c_1,c_2>0$.
\end{remark}

\begin{definition} \label{DNF hypothesis}
Denote by $\cH_{\DNF}$ the hypothesis class of all functions that can be realized by a $\DNF$ with a small number of terms. That is $$\cH_{\DNF} = \{h_F : \text{ F is a } \DNF \text{ formula with at most } n^2 \text{ terms}\}$$
\end{definition}

Intuitively, when evaluating a DNF formula on a given example, we check a few conditions (corresponding to the formula's terms), and deem the example positive if one of the conditions holds.
We will consider the case that for each of these conditions, there is some chance to see a ``prototype example". Namely, an example that satisfies only this condition in a strong (or evident) way.

\begin{definition} 
\label{satisfies evidently}
Let $F=T_1 \vee T_2 \vee \ldots \vee T_d$ be a DNF formula. An example $\x\in \{-1,1\}^n$ satisfies a term $T_i$ (with respect to the formula $F$) {\bf evidently} if :
\begin{itemize}
\item
It satisfies $T_i$. (In particular, $h_F(\x)=1$)
\item
It does {\bf not} satisfy any other term $T_k$ (for $k\neq i$) from F.
\item
No coordinate change will turn $T_i$ False and another term $T_k$ True.
Concretely, if for $j\in [n]$ we denote $\x^{\oplus j} = (x_1, \ldots, x_{j-1}, -x_j, x_{j+1}, \ldots ,x_n)$, then for every coordinate $j\in [n]$, if $\x^{\oplus j}$ satisfies $F$ (i.e. if $h_F(\x^{\oplus j} )=1$) then $\x^{\oplus j}$ satisfies $T_i$ and only $T_i$. 
 \end{itemize}
\end{definition} 
The first distributional assumption that we consider is that each positive example satisfies one term evidently.

\begin{definition} 
\label{distribution realized by DNF with evident}
A pair $(\cD,h^\star)$ of a distribution $\cD$ over $\{-1,1\}^n$ and $h^\star : \{-1,1\}^n \to \{0,1\}$ is {\textbf{realized by a small DNF with evident examples}} if there exists a DNF formula $F=T_1 \vee T_2 \vee \ldots \vee T_d$ over $\{-1,1\}^n$ with $d \leq n^2$ such that $h^\star = h_F$ and additionally, every positive example  $ \x\in \{-1,1\}^n$ with $\cD(\x)>0$ satisfies one of $F$'s terms evidently.
\end{definition} 

 One of the assumptions in our definition is that the target function can be realized by a $\DNF$ formula for which every example satisfies at most one term. For a function that is realized by a decision tree this always holds. So, in a sense, our assumption holds for functions that can be realized by a ``stable" decision tree.
 
The above definition makes a strong assumption, namely that \textit{every} positive example is an evidence for one term. The next definition relaxes that assumption and only assumes that for every term there is a non-negligible probability to see an evident example. 

\begin{definition} 
\label{distribution weakly realized by DNF with evident }
A pair $(\cD,h^\star)$ of a distribution $\cD$ over $\{-1,1\}^n$ and $h^\star : \{-1,1\}^n \to \{0,1\}$ is {\textbf{weakly realized by a small $\DNF$ with evident examples}} if there exists a DNF formula $F=T_1 \vee T_2 \vee \ldots \vee T_d$ over $\{-1,1\}^n$  with $d \leq n^2$ such that $h^\star = h_F$ and for every term $T_i$ there is a non-negligible\footnote{Recall that non-negligible is at least $\frac{1}{n^3}$  } probability to see an example that satisfies this term evidently.
\end{definition}

For example, our assumption holds for every distribution $\cD$, provided that $h^\star$ can be realized by a DNF formulas in which any pair of different terms contains two opposite literals.

\subsection{Upper Bounds}

We will now present two learning algorithms that use 1-LQ, and prove that each of these algorithms learn the class  $\cH_{\DNF}$ with respect to the families of distributions defined above. Both algorithms use the following claim that follows directly from definition~\ref{satisfies evidently}

\begin{claim}\label{propeq}
Let $F=T_1 \vee T_2 \vee \ldots \vee T_d$ be a $\DNF$ formula over $\{-1,1\}^n$. 
Then for every $\x \in\{-1,1\}^n$ that satisfies a term $T_i$ evidently (with respect to $F$), for every $j \in [n]$ it holds that:
$$
h_F(\x^{\oplus j} )=1 \Longleftrightarrow  \text{ the term $T_i$ does not contain the variable } x_j
$$
\end{claim}

\begin{algorithm}[h]
\caption{ Create a $\DNF$ formula }
\textbf{Input: } $S \in (\{-1,1\}^n  \times\{0,1\})^m $\\
\textbf{Output: } A $\DNF$ formula $H$
\begin{algorithmic}\label{alg1}

\STATE start with an empty $\DNF$ formula $H$

\FORALL {$(\x,y) \in S$}
\IF {$y=1$}
\STATE define $T = x_1\wedge \overline{x_1}\wedge x_2\wedge \overline{x_2}\wedge \ldots \wedge x_n \wedge \overline{x_n} $ \FOR{$ 1\leq j \leq n$}
\STATE query $\x^{\oplus j}$ (to get $h^\star(\x^{\oplus j})$)
\IF{$h^\star(\x^{\oplus j})=1$}
\STATE remove $x_j$ and $\overline{x_j}$ from $T$
\ENDIF
\IF{$h^\star(\x^{\oplus j})=0$}
\IF{$x_j=1$}
 \STATE remove $\overline{x_j}$ from $T$
\ENDIF
\IF{$x_j=0$}
 \STATE remove $x_j$ from $T$
\ENDIF
\ENDIF

\ENDFOR

\ENDIF

\STATE $H = H \vee T$

\ENDFOR
\RETURN{H}

\end{algorithmic}
\end{algorithm}

\begin{theorem}
\label{learnable evident dist} 
The hypothesis class $\cH_{\DNF}$ is 1-LQ learnable with respect to distributions that are realized by a $\DNF$ with evident examples.
\end{theorem}

\begin{proof}
We will prove that algorithm~\ref{alg1} learns $\cH_{\DNF}$ with 1-local membership queries.
First, it is easy to see that this algorithm is efficient: For a training set of size $m$ the algorithm asks for at most $n\cdot m$ $1$-local membership queries, and runs in time $O(nm)$. Likewise, the hypothesis that the algorithm returns is a $\DNF$ formula with at most m terms and every term is of size at most n, therefore it can be evaluated in time polynomial in $mn$. 

Now, let $\cD$ be a distribution on $\{-1,1\}^n$ and $h^\star :  \{-1,1\}^n \to  \{0,1\}$ be a hypothesis such that the pair $(\cD, h^\star)$ is realized by a small $\DNF$ with evident examples. Let $F=T_1 \vee T_2 \vee \ldots \vee T_d$ be that small $\DNF$ formula, (in particular  $h^\star = h_F$ and $d \leq n^2$). For $\epsilon > 0$ we take a sample $S = \{(\x_i,h^\star(\x_i)\}_{i=1}^m$ where $\{\x_i\}_{i=1}^m$ are sampled i.i.d from 
$\cD$ and  $m = \frac{2n^2}{\epsilon} \log \frac{2n^2}{\epsilon} \geq \frac{2d}{\epsilon} \log \frac{2d}{\epsilon}$.

Let $H$ be the $\DNF$ formula returned by the algorithm after running on $S$, and let $\hat{h}$ be the function induced by $H$. We will prove that with probability of at least 3/4 (over the choice of the examples) $L_{\cD,h\star}(\hat{h}) < 4\epsilon$. 

From the assumption on the distribution we get that every instance $\x$ that satisfies the formula (in our case every $\x$ such that $(\x,1)\in S$), satisfies exactly one term $T$. For every one of these positive instances from $S$, we will show that we add that exact term to $H$. For every such $\x$ we start with a full term (containing all the possible literals) and then for every $j \in [n]$, at iteration $j$:
\begin{itemize}
\item
if $h^\star(\x) = h^\star(\x^{\oplus j} )=1$ we know from claim~\ref{propeq} that the variable $x_j$ cannot appear in $T$ - so we remove it and its negation from the current term.
\item
if $h^\star(\x) = 1$ and $h^\star(\x^{\oplus j} )=0$ we know that either $x_j$ or $\overline{x_j}$ appears in $T$ and we remove the one that cannot appear in $T$ according to the value of $x_j$.
\end{itemize}
After $n$ iterations we get exactly $T$ - the term that $\x$ satisfies evidently. 
Therefore - $H$ will contain every term from $F$ for which there was an instance $\x$ in $S$ that satisfies it - other then that $H$ will contain no other terms. In other words, $$\mathop{\prob}_{\x \sim \cD}[h^\star(\x)=0 \wedge \hat{h}(\x)=1] = 0$$
and we get that 
\begin{gather*}
L_{\cD,h^\star}(\hat{h}) = \mathop{\prob}_{\x \sim \cD}[h^\star(\x) \neq \hat{h}(\x)] = \mathop{\prob}_{\x \sim \cD}[h^\star(\x)=1 \wedge \hat{h}(\x)=0]
\end{gather*}
Denote by $p_i$ the probability to sample $\x$ (from $\cD$) that will satisfy $T_i$, and let $A_i$ be the event that $S$ did not contain any $\x$ which satisfies $T_i$. Then 
\begin{gather*}
\mathop{\prob}_{\x \sim \cD}[h^\star(\x)=1 \wedge \hat{h}(\x)=0] = \mathop{\prob}_{\x \sim D}[\exists i \in [d] \text{ such that } \x \text{ satisfies } T_i \wedge \hat{h}(\x)=0]  \\
\leq \sum_{i=1}^d \mathop{\prob}_{\x \sim \cD}[ \x \text{ satisfies } T_i \wedge \hat{h}(\x)=0]  = \sum_{i=1}^dp_i \cdot \mathbbm{1}_{A_i}
\end{gather*}
Notice that since $p_i$ is the probability to sample $x$ we get that $\underset{S \sim D^m}{\prob}[A_i] = (1-p_i)^m$ 
\newpage
Now if we look at the expectation we get

\begin{eqnarray*}
\mathop{\bE}_{S \sim \cD^m}[L_{\cD,h^\star}(\hat{h})] &\leq& \mathop{\bE}_{S \sim \cD^m}[\sum_{i=1}^d p_i \cdot \mathbbm{1}_{A_i}]
\\
&=& \sum_{i=1}^d p_i \mathop{\bE}_{S \sim D^m}[\mathbbm{1}_{A_i}]
\\
&=& \sum_{i=1}^d p_i \mathop{\prob}_{S \sim D^m}[A_i]
\\
&=& \sum_{i=1}^d p_i (1-p_i)^m
\\
&=&\sum_{i | p_i<\frac{\epsilon}{2d}} p_i (1-p_i)^m + \sum_{i | p_i\geq\frac{\epsilon}{2d}} p_i  (1-p_i)^m
\\
&\leq & \sum_{i | p_i<\frac{\epsilon}{2d}} \frac{\epsilon}{2d} + \sum_{i | p_i\geq\frac{\epsilon}{2d}}  (1-p_i)^m
\\
&\leq & d \cdot \frac{\epsilon}{2d} + \sum_{i | p_i\geq\frac{\epsilon}{2d}} e^{-mp_i}
\\ &\leq&  \frac{\epsilon}{2} + d \cdot e^{-m\frac{\epsilon}{2d}} 
\end{eqnarray*}

Since $m \geq \frac{2d}{\epsilon} \log \frac{2d}{\epsilon}  $  we get $\bE[L_{\cD,h^\star}(\hat{h})] < \epsilon $ and using Markov's inequality we obtain

\begin{gather*}
\mathop{\prob}_{S \sim D^m}[L_{\cD,h^\star}(\hat{h})] \geq 4\epsilon] \leq  \frac{\bE [L_{\cD,h^\star}(\hat{h})]}{4\epsilon} < \frac{1}{4}
\end{gather*}

\end{proof}

\newpage
\begin{algorithm}[h]
\caption{ Create a $\DNF$ formula with checking and deleting false terms}
\textbf{Input: } $S_1 , S_2 \subseteq (\{-1,1\}^n  \times\{-1,1\})^m $\\
\textbf{Output: } a $\DNF$ formula $H$
\begin{algorithmic}\label{alg2}

\STATE start with an empty $\DNF$ formula $H$

\FORALL {$(\x,y) \in S_1$}
	\IF {$y=1$}
		\STATE define $T = \x_1\wedge \overline{x_1}\wedge x_2\wedge \overline{x_2}\wedge \ldots
		\wedge x_n 	\wedge \overline{x_n} $ \FOR{$ 1\leq j \leq n$}
		\STATE query $\x^{\oplus j}$ (to get $h^\star(\x^{\oplus j})$)
		\IF{$h^\star(\x^{\oplus j})=1$}
			\STATE remove $x_j$ and $\overline{x_j}$ from $T$
		\ENDIF
		\IF{$h^\star(\x^{\oplus j})=0$}
			\IF{$x_j=1$}
 				\STATE remove $\overline{x_j}$ from $T$
			\ENDIF
			\IF{$x_j=0$}
 				\STATE remove $x_j$ from $T$
			\ENDIF
		\ENDIF

\ENDFOR
	\STATE $H = H \vee T$
	\ENDIF

\FORALL{$T$ in $H$}
\FORALL {$(\x,y) \in S_2$}
\IF{$T(\x)=1$ but $y=0$}
 \STATE remove $T$ from $H$
\ENDIF
\ENDFOR

\ENDFOR
\ENDFOR
\RETURN{H}

\end{algorithmic}
\end{algorithm}

\begin{theorem}
\label{learnable weak evident dist} 
The hypothesis class $\cH_{\DNF}$ is 1-LQ learnable with respect to distributions that are weakly realized by a $\DNF$ with evident examples. \end{theorem}

\begin{proof}
We will prove that algorithm~\ref{alg2} learns  $\cH_{\DNF}$ with 1-local membership queries. In this case we will have two sample sets - $S_1$ of size $m_1$ which will be used as before - to build the terms of $H$, and $S_2$ of size $m_2$ - a separate set to check the terms that were built. 
Again, it is easy to see that this algorithm is efficient. For training sets $S_1$ of size $m_1$ and $S_2$ of size $m_2$ the algorithm asks for at most $n\cdot m_1$ $1$-local membership queries. The running time of the first loop is $O(nm_1)$ and in that loop we add at most $m_1$ terms to $H$ so the running time of the second loop is $O(m_1m_2)$. All in all the running time is polynomial in $(m_1,m_2,n)$. Also, the hypothesis that the algorithm returns is a $\DNF$ formula with at most $m_1$ terms and every term is of size at most n, therefore it can be evaluated at time polynomial in $m_1n$.

Now, let $\cD$ be a distribution on $\{-1,1\}^n$ and $h^\star :  \{-1,1\}^n \to  \{0,1\}$ be a hypothesis such that the pair $(\cD, h^\star)$ is realized by a small $\DNF$ with evident examples. Let $F=T_1 \vee T_2 \vee \ldots \vee T_d$ be that small $\DNF$ formula, (in particular  $h^\star = h_F$ and $d \leq n^2$). Denote by $H=\hat{T}_1 \vee \hat{T}_2 \vee \ldots \vee \hat{T}_k$ the DNF formula algorithm~\ref{alg2} returns. Following the same argument from the last proof, a term $T_i$ will be added to H in the first loop if $S_1$ contains an example that satisfies $T_i$ evidently. We will define $m_1$ so that with high probability for every term $T_i$ there will be $ (\x,1) \in S_1$ such that $\x$ satisfies $T_i$ evidently. \newline
Denote by $s_i$ the probability to sample $\x$ (from $\cD$) that satisfies $T_i$ evidently, and let $s = \min\{s_i\}_{i=1}^d$. Since for every term the probability to see an evident example is non-negligible, $s\ge n^{-3}$.
For every i, the probability of \textit{not} seeing an example in $S_1$ that satisfies $T_i$ evidently is 
\begin{gather*}
(1-s_i)^m \leq (1-s)^m \leq e^{-sm} \leq e^{-\frac{m}{n^3}}
\end{gather*}
If we set $m_1$ to be $n^3\log(8n^2)\ge n^3\log(8d)$ we get that the probability of not seeing an example that satisfies $T_i$ evidently (when sampling $S_1$ from $\cD^{m_1}$) is less than $\frac{1}{8d}$ and from the union bound we get that the probability that the sample will contain an evident example for every term is at least $\frac{7}{8}$. Therefore with probability of at least $\frac{7}{8}$ we will add every $T_i$ to $H$ in the first loop. In the second loop, when we remove terms from $H$, we only remove terms which contradicts one of the examples in $S_2$. Since all of the examples in the sample set are labeled by $F$, we will never remove a term that is a part of $F$ 
Therefore with probability of at least $\frac{7}{8}$ $H$ will contain \textit{all} of $F$'s terms. Formally, 
\begin{gather*}
\mathop{\prob}_{S_1 \sim \cD^{m_ 1}} [\mathop{\prob}_{\x \sim \cD}[h^\star(\x)=1 \wedge \hat{h}(\x)=0] = 0] \geq \frac{7}{8}
\end{gather*}

Note that we are not done, as the algorithm might create a wrong term (when using a "non-evident" example). For this reason we add the second loop. We use the sample $S_2$ to test every term $\hat{T}_i$ that was added to $H$ in the first loop. If we see an example $\x$ such that $\hat{T}_i(\x)=1$ but $h^\star(\x)=0$ we remove $\hat{T}_i$ and continue to the next term. Now denote by $p_i$ the probability to sample $\x$ (from $\cD$) that will satisfy $\hat{T}_i$, and by $A_i$ the event that $\hat{T}_i$ is a wrong term (not from F) but the ''checking" step did not discover that. Then

\begin{eqnarray*}
\mathop{\prob}_{\x \sim \cD}[\hat{h}(\x)=1 \wedge h^\star(\x)=0] &=& \mathop{\prob}_{\x \sim \cD}[\exists i \in [k] \text{ such that } \x \text{ satisfies } \hat{T}_i \wedge h^\star(\x)=0] 
\\
&\leq & \sum_{i=1}^k \mathop{\prob}_{\x \sim \cD}[\x \text{  satisfies } \hat{T}_i \wedge h^\star(\x)=0]
\\
&=& \sum_{i=1}^kp_i \cdot \mathbbm{1}_{A_i}
\end{eqnarray*}
Note that since $A_i$ is the event that there wasn't any example in $S_2$ which satisfied $\hat{T}_i$ (otherwise the checking step would discover that $\hat{T}_i$ is wrong) this is the same situation as in the proof of theorem~\ref{learnable evident dist}, so
\begin{gather*}
\mathop{\prob}_{S_2 \sim \cD^{m_2}}[A_i] = (1-p_i)^{m_2}
\end{gather*}
By the same analysis of the former proof, we get that if the size of $S_2$ is $\ge  \frac{2k}{\epsilon} \log \frac{2k}{\epsilon}$ then
\begin{gather*}
\mathop{\prob}_{S_2 \sim \cD^{m_2}}[\mathop{\prob}_{\x \sim \cD}[h^\star(\x)=0 \wedge \hat{h}(\x)=1] \geq 4\epsilon] \leq \frac{1}{4}
\end{gather*}
Finally we notice that $k \leq m_1$, because for each example in $S_1$ the algorithm adds at most one term to $H$. So we can set $m_1$ as above and $m_2 =  \frac{2m_1}{\epsilon} \log \frac{2m_1}{\epsilon}  $ and if we run algorithm ~\ref{alg2} on $S_1$ and $S_2$ we get that with probability of at least $1 - (\frac{1}{4} + \frac{1}{8}) = \frac{3}{4} - \frac{1}{8}$ over sampling $S_1$ and $S_2$ 
\begin{eqnarray*}
L_{\cD,h^\star}(\hat{h}) &=& \mathop{\prob}_{\x \sim \cD}[h^\star(\x) \neq \hat{h}(\x)]
\\
&=& \mathop{\prob}_{\x \sim \cD}[h^\star(\x)=1 \wedge \hat{h}(\x)=0] + \mathop{\prob}_{\x \sim \cD}[h^\star(\x)=0 \wedge \hat{h}(\x)=1]
\\
&\leq& 0 + 4\epsilon = 4 \epsilon
\end{eqnarray*}

\end{proof}

\subsection{A Lower Bound}

In this section we provide evidence that the use of queries in our upper bounds is crucial. We will show that the problem of learning poly-sized decision trees can be reduced to the problem of learning DNFs w.r.t. distributions that are realized by a small DNF with evident examples. As learning decision trees is widely believed to be intractable (in fact, even learning the much smaller class of $\log(n)$-juntas is conjectured to be hard), this reduction serves as an indication that the problems we considered are hard without membership queries.

\begin{definition}
A decision tree over $\{-1,1\}^n$ is a binary tree with labels chosen from \\$x_1 , \ldots , x_n $ on the internal
nodes, and labels from $\{0, 1\}$ on the leaves. Each internal node's left branch is viewed as
the $-1$ branch; the right branch is the $1$ branch. Each decision tree over $n$ variables induces a function $h : \{-1,1\}^n \to \{0,1\}$ in the following way: For a decision tree $T$ , a vector $\textbf{a} \in \{-1, 1\}^n$ defines a path in the tree from the root to a specific leaf by choosing $a_i$'s branch at each node $x_i$ and the value that the function $h_T$ returns on $\textbf{a}$ is defined to be the label of the leaf at the end of this path.
\end{definition}

\begin{definition} \label{DT hypothesis}
Denote by $\cH_{\DT}$ the hypothesis class of all functions that can be realized by a decision tree with a small number of leaves. That is $$\cH_{\DT} = \{h_T : \text{ T is a } \DT \text{  with at most } n^2 \text{ leaves}\}$$
\end{definition}

\begin{theorem}\label{hardness}
PAC learning  the hypothesis class $\cH_{\DNF}$ w.r.t distributions that are realized by a small $\DNF$ with evident examples is as hard as PAC learning $\cH_{\DT}$.
\end{theorem}

The proof will follow from the following claim:
\begin{claim} \label{reduction from DT to DNF}
There exists a mapping (a reduction) $\varphi : \{-1,1\}^n \rightarrow \{-1,1\}^{2n} $, that can be evaluated in $poly(n)$ time so that for every decision tree $T$ over $\{-1,1\}^n$ there exists a $\DNF$ formula $F$ over $\{-1,1\}^{2n}$ such that the following holds:
\begin{enumerate}
\item
The number of terms in $F$ is upper bounded by the number of leaves in $T$
\item
$h_T = h_F \circ \varphi$
\item
$\forall \x$ such that $h_T(\x) = 1$ ,  $\varphi (\x)$ satisfies some term in $F$ evidently. 
\end{enumerate}
\end{claim}

\begin{proof}
We will denote $\{-1,1\}^n$ by $\cX_n$ and $\{-1,1\}^{2n}$ by $\cX_{2n}$.

Define $\varphi$ as follows: 
$$ \forall \x=(x_1,x_2, \ldots , x_n) \in \cX_n  \qquad \varphi(x_1,x_2, \ldots , x_n) = (x_1,x_1,x_2, x_2 , \ldots , x_n, x_n) $$
Now, for every tree $T$, we will build the desired $\DNF$ formula $F$ as follows: 
First we build $F'$ - a $\DNF$ formula over  $\{-1,1\}^n$ . Every leaf labeled '$1$' in $T$ will define the following term- take the path from the root to that leaf and form the logical AND of the literals describing the path. $F'$ will be a disjunction of these terms. Now, for every term $T$ in $F'$ we will define a term $\phi(T)$ over $\cX_{2n} $ in the following way: Let $P_T = \{i \in [n] : x_i \textit{ appear in T} \}$ and $N_T = \{i \in [n] : \overline{x_i} \textit{ appear in T} \}$. So 
$$ T = \bigwedge_{j \in P_T} x_j \bigwedge_{j \in N_T} \overline{x_j}$$ 
Define 
\begin{gather*}
\phi(T) = \bigwedge_{j \in P_T} x_{2j-1} \bigwedge_{j \in P_T} {x_{2j}}
\bigwedge_{j \in N_T}\overline{x_{2j-1}} \bigwedge_{j \in N_T}  \overline{x_{2j}} \\
 \end{gather*}
Finally, define $F$ to be the $\DNF$ formula over $\cX_{2n}$ by
$$ F =  \bigvee_{T \in F'} \phi(T)$$

We will now prove that $\varphi$ and $F$ satisfy the required conditions.
First, $\varphi$ can be evaluated in linear time in $n$.
Second, it is easy to see that $h_T = h_F \circ \varphi$, and as every term in $F$ matches one of $T$'s leaves, the number of terms in $F$ cannot exceed the number of leaves in $T$. It is left to show that the third requirement holds. Let there be an $\x$ such that $h_T(\x) = 1$, then $x$ is matched to one and only one path from $T$'s root to a leaf labeled '1'. From the construction of $F$, $\x$ satisfies one and only one term in $F'$ because every term is matched to exactly one path from $T$'s root to a leaf labeled 1.
Regarding the last requirement - that no coordinate change will make one term from $F$ False and another one True - we made sure this will not happen by ``doubling" each variable. By this construction, in order to change a term from False to True at least two coordinate must change their value.

\end{proof}

\begin{proof} [of theorem~\ref{hardness}]
Suppose we have an efficient algorithm $\cA$ that PAC learns $\cH_{\DNF}$ with respect to distributions that are realized by $\DNF$ with evident examples. Using the reduction from claim~\ref{reduction from DT to DNF} we will build an efficient algorithm $\cB$ that will PAC learn $\cH_{\DT}$.\newline
For every training set with examples from $\cX_n$:
\[
S = \{(\x_1,h^\star(\x_1)) , (\x_2,h^\star(\x_2)) , \ldots , (\x_m,h^\star(\x_m)) \} \in (\cX_n \times \{0,1\})^m
\]
we define a matching training set with examples from $\cX_{2n}$, using $\varphi$ from the above claim:
\[
\tilde{S} := \{(\varphi(\x_1),h^\star(\x_1)) , (\varphi(\x_2)),h^\star(\x_2)) , \ldots , (\varphi(\x_m),h^\star(\x_m)) \} \in (\cX_{2n} \times \{0,1\})^m
\]
The algorithm $\cB$ will work as follows: \\ Given a training set $S$, $\cB$ will construct $\tilde{S} = \varphi(S)$ and then run $\cA$ with input $\tilde{S}$. Let $\hat{h}$ be the output of $\cA$ when running on $\tilde{S}$, $\cB$ will return $\hat{h} \circ \varphi$. Since $\varphi$ can be evaluated in $poly(n)$ time and $\cA$ is efficient, we get that $\cB$ is also efficient.

We will prove that algorithm $\cB$ is a learning algorithm for the class $\cH_{\DT}$. Since $\cA$ is a learning algorithm for the class $\cH_{\DNF}$ with respect to distributions that are realized by a small $\DNF$ with evident examples, there exists a function $m_\cA \left(n,\epsilon\right)\le \poly\left(n,\frac{1}{\epsilon}\right)$, such that 
for every $(\cD, h^\star)$ that is realized by a small $\DNF$ with evident examples and every $\epsilon>0$, if $\cA$ is given a training sequence
\[
S = \{(\x_1,h^\star(\x_1)) , (\x_2,h^\star(\x_2)) , \ldots , (\x_m,h^\star(\x_m)) \} 
\]
where the $\x_i$'s are sampled i.i.d. from $\cD$ and $m \ge m_\cA (n,\epsilon)$, then with probability of at least $\frac{3}{4}$ (over the choice of $S$), the output $\hat{h}$ of $\cA$ satisfies 
$L_{\cD, h^\star}(\hat{h}) \le \epsilon$. 

Let $\cD$ be a distribution on $\cX_n$ and let $h_T$ be a hypothesis that can be realized by a small $\DT$. Define a distribution $\tilde{\cD}$ on $\cX_{2n}$ by,
\[
\tilde{(\cD)}(\z) =
\begin{cases} 
\cD(\x) & \textrm{if}~ \exists \x \in \cX_{n} \textrm{ such that}~ \z = \varphi(\x) \\
0 & \textrm{otherwise}~ 
\end{cases}
\] 
Since $\varphi$ is one-to-one, $\tilde{\cD}$ is well defined and is a valid distribution on $\cX_{2n}$.

Now, as $h_T$ is realized by a small $\DT$, then from the conditions that $\varphi$ satisfies we get that there exists a $\DNF$ formula $F$ such that $h_T= h_F \circ \varphi$  and the pair $(\tilde{\cD}, h_F)$ is realized by a small $\DNF$ with evident examples.
Now  for every $\epsilon>0$ we take a sample
$
S = \{(\x_1,h_T(\x_1)) , (\x_2,h_T(\x_2)) , \ldots , (\x_m,h_T(\x_m)) \} $ with $m =  m_\cA (2n,\epsilon)$
and obtain that with probability of at least $\frac{3}{4}$ it holds that

\begin{eqnarray*}
L_{\cD,h_T}(\cB(S)) &=& L_{\cD,h_T}(\hat{h} \circ \varphi)
\\
&=& \mathop{\prob}_{\x \sim \cD}[h_T(\x) \neq \hat{h} \circ \varphi(\x)]
\\
&=& \mathop{\prob}_{\x \sim \cD}[h_F \circ \varphi (\x) \neq \hat{h} \circ \varphi(\x)] 
\\
&=& \mathop{\prob}_{\z \sim \tilde{\cD}}[h_F(\z) \neq \hat{h}(\z)]
\\
&=& L_{\tilde{\cD},h_F}(\hat{h})
\\ 
&=& L_{\tilde{\cD},h_F}(\cA(\tilde{S})) < \epsilon
\end{eqnarray*}
So $\cB$ is indeed a learning algorithm for the class $\cH_{\DT}$

\end{proof}
\newpage
\section{Experiments} \label{empirical}
Membership queries are a mean by which we can use human knowledge for improving performance in learning tasks.
Human beings have a very rich knowledge and understanding of many problems that the ML community works on. They can provide much more information than merely the category of the object or an answer to a ``yes" or ``no" question. This knowledge is often basic, and can be acquired without the use of an expert (e.g., using crowd-sourcing).
In this section we will present empirical results of an algorithm which takes advantage of this extensive knowledge in order to perform smart feature selection. 

In standard supervised classification tasks the user is only asked to give the label of each example. What we did in this task, is to ask for additional information.
Specifically, we faced a situation where we had a large number of features, and that these features had an interpretation that is easily understood. For every example in the sample set, we asked the user for its label and \textit{in addition}, we asked which features indicate that this instance is labeled as such. After we finished iterating over the entire sample, we used the information on the relevant features to narrow down the feature space. Concretely, we trained linear classifiers only on the features that were chosen to be indicative by the users.

Arguably, this algorithm gathers additional information in a manner that is similar to using 1-local membership queries. 1-local query tests whether changing the value of a single feature changes the label. This can be seen as asking whether this feature is relevant to the prediction or not. In the algorithm presented here, we ask for the relevant features in a broader way. Namely, we explicitly ask which words are relevant to the corresponding label.

\subsection{Is the additional data useful?}
When humans make decisions, it is often by very complex thought processes and we do not know whether we can access specific considerations that were used in the decision making process.  The first goal of this experiment is to show that at least for some tasks, important parts of this thought process are easily accessible. I.e., that the annotators' knowledge can be retrieved by asking simple questions.  
The second goal is to show that using this extra knowledge can help significantly decrease the number of tagged examples that are required.

We will formulate the above goals using the notion of \textit{error decomposition}. Let $\hat h$ be the classifier returned by the algorithm. We decompose $L_{\cD}(\hat h)$ as a sum of the {\em approximation error} (the error of the best linear classifier) and the {\em estimation error} (the difference between $L_{\cD}(\hat h)$ and the approximation error):
$$L_{\cD}(h_S) = \epsilon_{app} + \epsilon_{est} ~\text{ where } ~\epsilon_{app} = \min_{h \in \cH} L_{\cD}(h) ~\text{ and }~ \epsilon_{est} = L_{\cD}(h_S) - \min_{h \in \cH}L_{\cD}(h)$$
The approximation error $\epsilon_{app}$ measures how good is the class of linear classifiers that we restrict ourselves to. In other words, since the class is linear, how informative are the features we use. The estimation error measures to which extent the algorithm overfits the data.

We can now formulate the above goals into claims on the approximation and estimation error.
By applying the \textit{user induced} feature selection mentioned above we can only increase the approximation error, as we reduce the hypothesis class to a smaller one. We will want to show that the feature space chosen by the users is still expressive enough, so that the increase in the approximation error will be minor. In addition, we will show that the feature selection is effective in the sense that the estimation error decreases significantly.

\subsection{Experimental setup}
\subsubsection{Sentiment Analysis}
Sentiment analysis (SA) is the Natural
Language Processing task of identifying the attitude of a given text (usually whether it is positive, neutral or negative). This task has been studied in the NLP community for many years at different scale levels. It started off from being a document level classification task \citep{pang2004sentimental}, and then the focus shifted to handling the sentence level \citep{hu2004mining,kim2004determining}. The newest focus is sentiment analysis of Microblog data like Twitter.
Working with these informal text genres, on which users post their opnions, emotions, and recations about practically everything, presents
new challenges for natural language processing beyond those encountered when working with more
traditional text genres such as news-wire or product reviews.
Indeed, classical approaches to Sentiment Analysis \citep{pang2008opinion} are not directly applicable to tweets. While most of them focus on relatively large texts, e.g. movie or product reviews, tweets are very short and fine-grained.
Nevertheless, the great prominence of Social Media during the last few years encouraged a focus on
the sentiment detection over a microblogging domain. There has been a lot of recent work on sentiment analysis of twitter data. Some examples are \citep{pak2010twitter,kouloumpis2011twitter,davidov2010enhanced,barbosa2010robust}.

We chose this task to demonstrate our method since each example (tweet) is constructed from a limited number of features (words), making each of these features very important for classification. Therefore, it seems that information supplied by users, can be useful in focusing our attention on the important features.
Secondly, if in fact the two claims above hold, it will enable us to use a smaller data set, which is very important for this kind of tasks, since SA (and many more NLP tasks) require a large labeled data set which is often costly.

\subsubsection{Dataset}
\begin{table}[h]
\centering
\begin{tabular}{|c|c|c|c|c|c|}
\hline
  &  Negative & Neutral & Positive & All \\
\hline
Train& 1234 & $4193$ & 3012  & 8439 \\
\hline
Test& $640$ & $1962$ & $2099$ & 4701\\

\hline
\end{tabular}
\caption{The SemEval dataset}
\label{tab:semEval_data}
\end{table}
We worked with the data set from SemEval  \citep{nakov2013semeval}, a shared task for Sentiment Analysis of Tweets . This dataset is constructed of 13,140 (8,439 train+development and 4,701 test, see Table \ref{tab:semEval_data}) tweets which
were collected over a one-year period spanning from
January 2012 to January 2013. The tweets were labeled using the crowd sourcing tool Amazon Mechanical Turk and the labels were filtered to get rid of spammers. 

For each sentence (tweet), the users were asked to indicate the overall sentiment of the sentence - positive, negative or neutral \footnote{The original labeling had 4 classes-[objective,
positive, negative, or neutral] but since the turkers tended to mix up between the objective and neutral, the two classes were combined in the final task.} and also to mark all the subjective (positive or negative) words/phrases in the sentence\footnote{This labelling procedure was originally intended to be used for two separate tasks. The first is, when given a tweet containing a marked instance of a word or a phrase, to identify the sentiment of that instance (i.e., whether the word is negative or  positive). The second is identifying the sentiment of the whole tweet (without using the marked words).}. The learning task that we worked on is classifying the sentiment of the entire sentence. Although we only want to predict the sentiment of the tweet, we use these two labellings to get one ``richer" labelled data-set. I.e., each instance in our training set holds additional information to its sentiment - which words/phrases in the sentence indicate a positive or negative sentiment.

\subsubsection{Pre-processing}
Beside simple text, tweets may contain URL
addresses, references to other Twitter users (appear as @$<$username$>$) or content tags (also called hashtags) assigned by the tweeter ($\#<$tag$>$).
During preprocessing, we performed the following standard manipulations:
\begin{itemize}
\item
Words were switched to lower case and punctuation marks were removed (apart from a fixed set of smileys)
\item
Every hyperlink was replaced by the meta-word URL
\item
Every word starting with $ @$, i.e. a username in twitter syntax, was replaced by the meta-word USR. 
\item 
The hashtag sign $'\#'$ was removed from every tag to get a simple word. For example \emph{$\#$perfect} was changed to \emph{perfect}. 
\end{itemize}

\subsubsection{Language Model}
We used the simple bag-of-words language models of n-grams (in our case unigrams, bigrams and trigrams). I.e., each tweet is represented as a sparse vector in $\{0,1\}^d$, where $d$ is the size of the dictionary and the $i$'th coordinate equals 1 if and only if the $i$'th word in the dictionary appears in the tweet.
We performed a standard cut-off of rare n-grams \footnote{without performing this cutoff, the results for the non-query variant are much worse}. 

\subsubsection{Scoring}
The results were evaluated on averaged $F1$ scores. This scoring function is used in the SemEval shared task, and overall a very common scoring function for NLP tasks.
The $F1$ score is the harmonic mean of Precision and Recall. Every label has it's $F1$ score. For the positive label, the Precision is the number of tweets that were correctly labeled as positive divided by the total number of tweets that were labeled as positive: $$P_{POS} = \frac{TP}{TP + FP}$$ 
The Recall of the positive label is the number of tweets that were correctly labeled as positive divided by the total number of positive tweets in the data: $$ R_{POS} = \frac{TP}{TP+FN}$$
The positive label $F1$-score is computed as follows:
$$F_{POS} = 2\frac{P_{POS} \cdot  R_{POS}}{P_{POS} + R_{POS}}$$
The negative label $F1$-score $F_{NEG}$ is computed similarly.
The final score that the results are evaluated on is the average of the above two: $$F1 = \frac{1}{2} (F_{POS} + F_{NEG})$$

\subsubsection{The algorithm}

We compare two variants for the feature space: using the entire feature space (after cutting off the rare n-grams), and using the "query acquired" feature space which contains only features that were selected by the users as positive or negative for some example. Information about the data and the number of features is given in table \ref{tab:feature_info}.
\begin{table}[h]
\centering
\begin{tabular}{|c|c|c|c|c|c|}
\hline
 & Unigrams & Bigrams & Trigrams \\
\hline
Overall number of features &18257& 89788& 128699\\
\hline
Features after cutoff& 3182 & 3099& 1718\\
\hline
Features selected by the users &1391& 1368&  846\\
\hline
\end{tabular}
\caption{Information about the features}
\label{tab:feature_info}
\end{table}

We used a simple Naive Bayes classifier, with a small smoothing parameter. We also checked other classification algorithms- random forests, logistic regression, and multiclass SVM, (with $\| \cdot\|_{1}$-regularization and $\| \cdot\|_{2}$-regularization), but the results of the Naive Bayes predictor were the highest for both feature spaces.

\begin{figure}[h]
\centering
\subfigure[F1-scores for positive samples]{
\includegraphics[width=.35\textwidth]{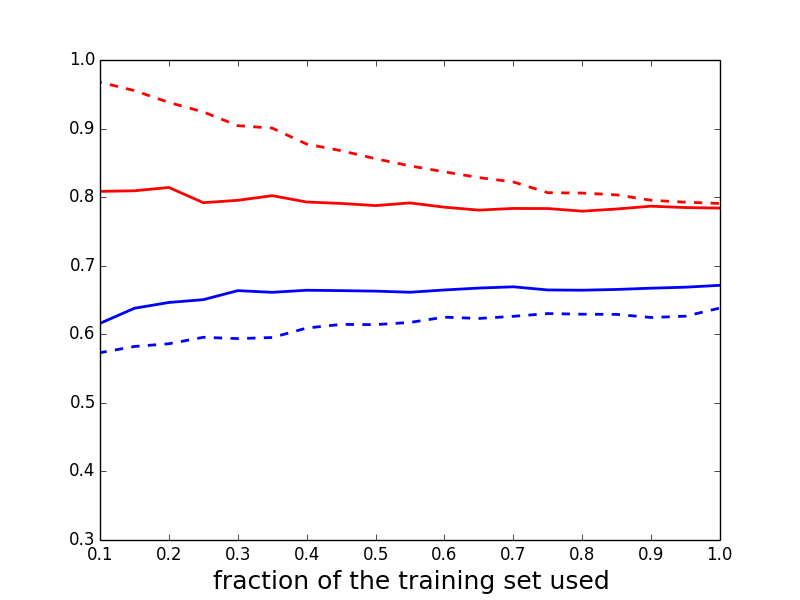}
}
\subfigure[F1-scores for negative samples]{
\includegraphics[width=.35\textwidth]{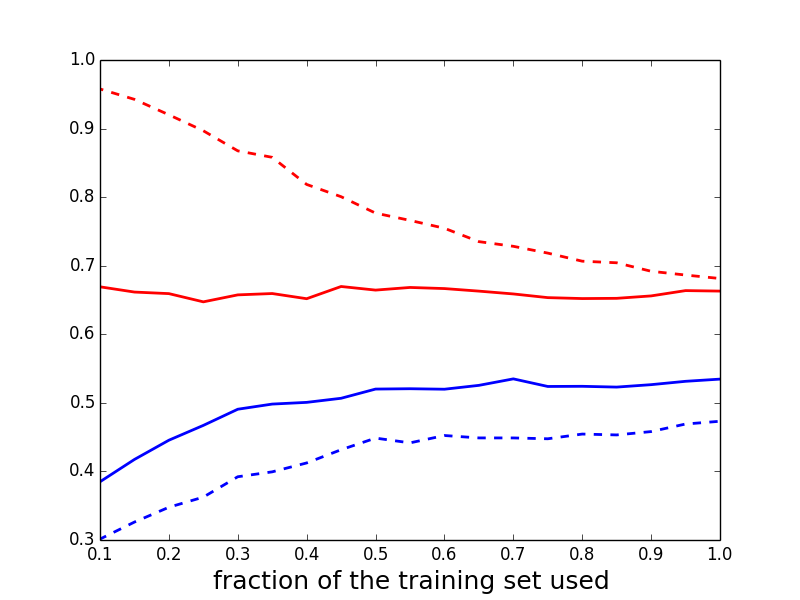}
}
\subfigure[Averaged F1-scores]{
\includegraphics[width=.6\textwidth]{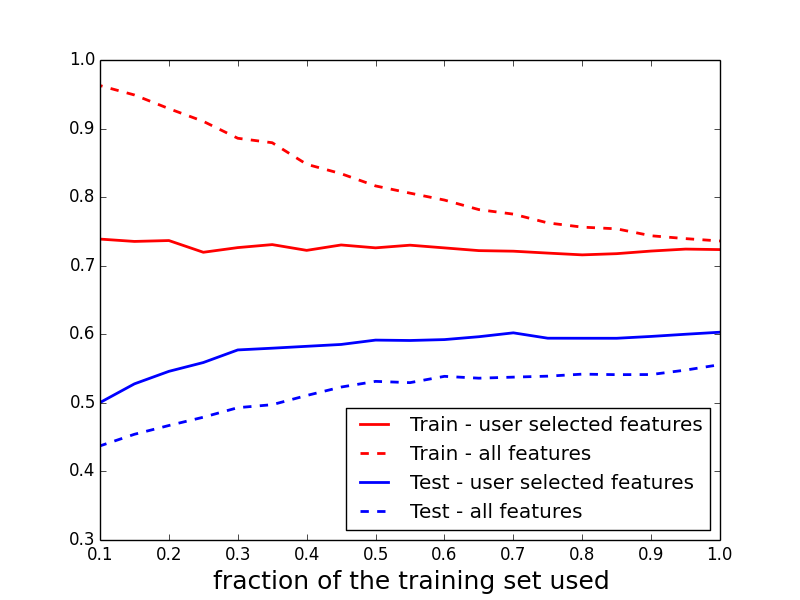}
}
\caption{Train (red) and test (blue) $F1$-scores for a Naive Bayes classifier using the entire feature space (dashed lines) compared to using the queries-acquired feature space (continues lines): (a) for positive samples, (b) for negative samples and (c) average of positive and negative scores.}
\label{fig:train_size_error}
\end{figure}

\subsection{Results}

The results that we will present are the results of the unigram model. The test scores of the other language models (unigram+bigrams and unigram+bigram+trigram) are almost identical for both feature spaces, and the training scores gets higher with the model complexity, as expected. Since our training set only contains approximately 8000 instances, we chose to present the results of the simplest model, so that the number of features would be comparable to the number of instances.

The results of both variants are presented in figure \ref{fig:train_size_error}. As can be seen by the test scores, our algorithm outperforms the other variant which does not uses the additional information. The difference in test performance is approximately constant across different training sizes.
Getting back to our claims - regarding the approximation error, by looking at the final training scores (using the larger training set possible), it can seen that both variants are almost identical in all of the measurements. This fact indicates that we did not increase the approximation error.
Regarding the improvement of estimation error, this can be seen clearly by looking at the gap between the test scores and the train scores. The gap in the query acquired model is smaller than the gap in the other model.

\subsubsection{Precision and Recall}
Additional interesting properties can be seen in the precision and recall graphs  (figure \ref{fig:precision_and_recall}). 
For example, by looking at the results for positive samples (a \& b) we can see that the improvement in the results from using the query model is almost only due to the improvement in the precision scores. If we only use $10\%$ of the data, the query model reaches 0.77 test precision, while the non-query model only reaches 0.71 test precision score even when using the whole data set. Another interesting property that can be seen is that when a small training set is used, the difference in the test scores between the query and non-query methods is about twice as large as the difference when the largest possible training set is used.

\begin{figure}[h]
\centering
\subfigure[Precision scores -- positive samples]{
\includegraphics[width=.35\textwidth]{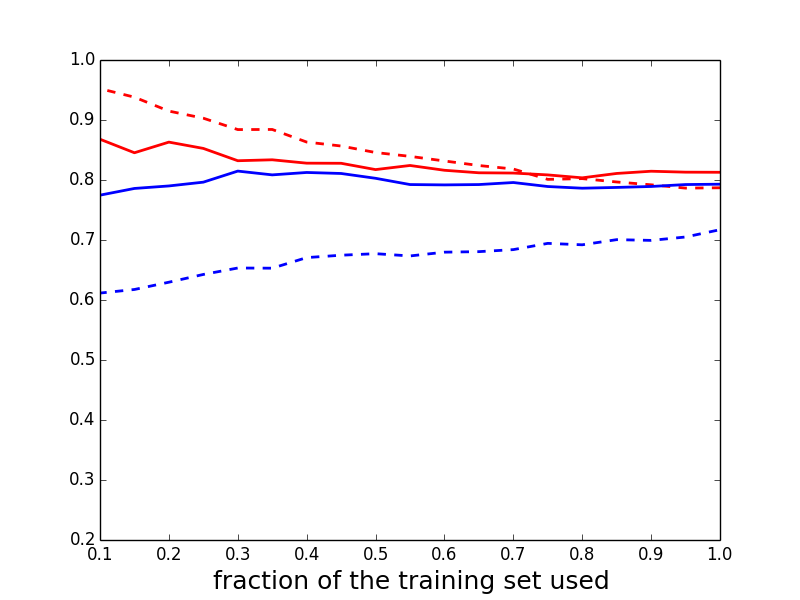}
}
\subfigure[Recall scores -- positive samples]{
\includegraphics[width=.35\textwidth]{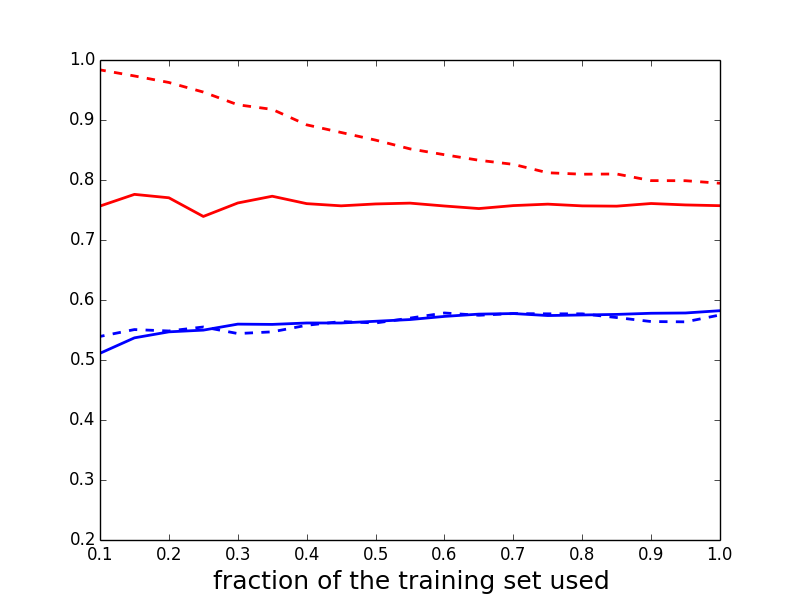}
}

\subfigure[Precision scores -- negative samples]{
\includegraphics[width=.35\textwidth]{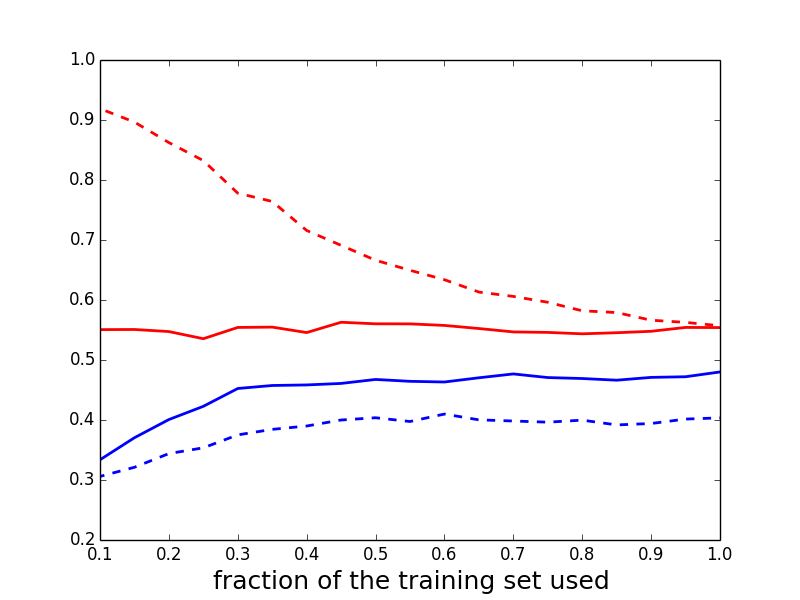}
}
\subfigure[Recall scores -- negative samples]{
\includegraphics[width=.35\textwidth]{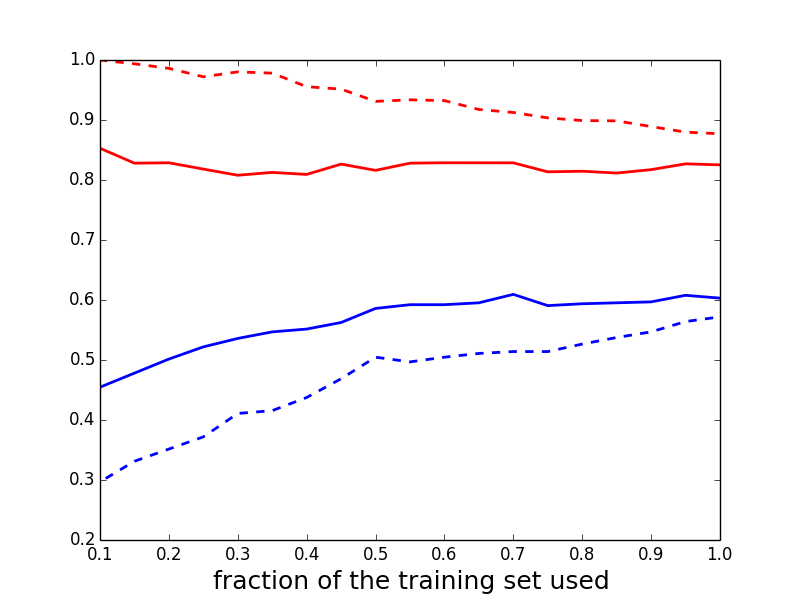}
}

\caption{Train (red) and test (blue) precision and recall scores for a Naive Bayes classifier using the entire feature space (dashed lines) compared to using the queries-acquired feature space (continues lines). Top -- positive samples, bottom -- negative samples, left -- precision and right -- recall.}
\label{fig:precision_and_recall}
\end{figure}

\subsubsection{Over-fitting}
When using the naive bayes algorithm, we estimate $\prob(f | c)$ for every feature $f$ and every label $c$. This term measures how much the appearance of $f$ contributes to the fact that $c$ is the correct label \footnote{by the naive assumption that all of the features are independent given the label, this information is actually the only information we use in order to build the classifier} . Using those terms, we can sort the features by an order which conveys their informativeness.
Since our features are words (or bigrams or trgrams), we can get some interesting insights by looking at the most informative features that each variant uses. If we only look at the top of the list (the top 20), the chosen features by both variants are almost identical. But, if we look a bit further we see how the algorithm which uses the entire feature space, chooses some significant features which clearly over-fit the training data. Some example are : "nick", "lloyd", and "justin" in the unigram model, "saturday kitchen", "ghost rider", "ray lewis" in the bigram model and "rugby world cup" in the trigram model.
 
This over-fitting will obviously decrease as we increase the training size (and practically by checking the most informative features at different training sizes, the smaller the sample is, the more easy it is to find over-fitting features like the above). But as already stated,
generally in Natural Language Processing it is much harder to acquire a large labeled data set. Therefore a method that avoids or significantly decrease this kind of over-fitting will be of high value.

\subsection{Comparing to other Feature Selection Methods}
A question that can be raised is whether the improvement in the results is just an effect of the feature selection itself, or that the fact that the features were selected by a query process is the important part. In order to answer this, we compared our algorithm to using other \textit{automatic} feature selection techniques.  
We checked two feature selection methods - filter and backward elimination. For each training set, the number of features that the method was instructed to select was the same as the number of features chosen by the users on that set. The results are presented in figure  \ref{fig:feature_selection_error}.  The training scores of the automatic feature selection techniques are much lower than the training score of using the entire feature space (and much more similar to those of our method already for small training sets). This fact is reasonable, as we use a much smaller hypothesis class. If we look at the test scores it can be seen that using other feature selection techniques does improve the test score a little when compared to no feature selection at all, but still lies well under the score of our query acquired features method.

Another feature selection method that we compared our results to was using a SVM classifier with $\| \cdot\|_{1}$-regularization, which is known to induce sparsity. Here again, using our query acquired feature set outperforms in all of the measurements. 

\begin{figure}[p]
\centering

\subfigure[Averaged F1 train scores]{
\includegraphics[width=.55\textwidth]{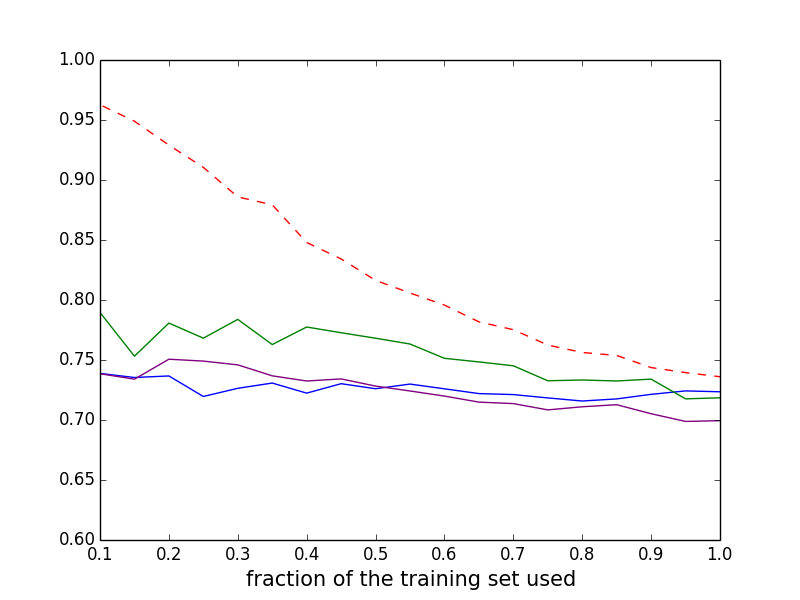}
}
\subfigure[Averaged F1 test scores]{
\includegraphics[width=.55\textwidth]{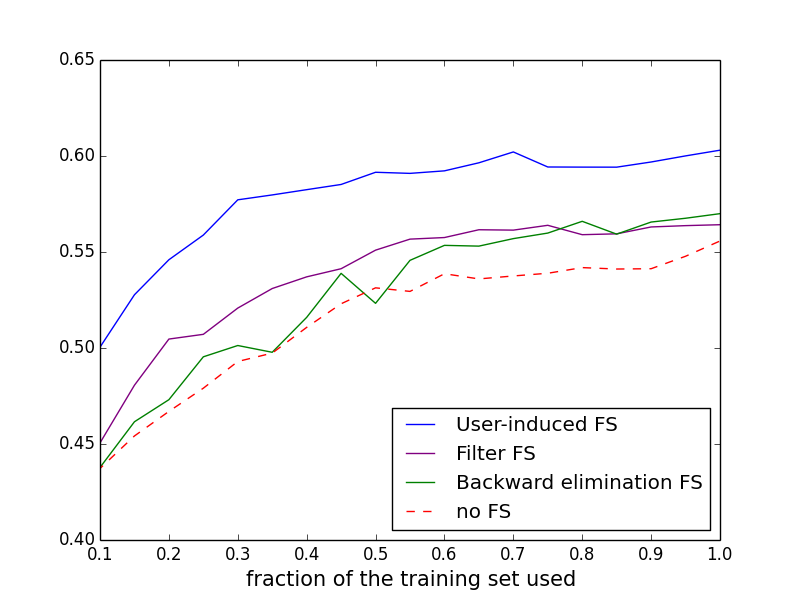}
}
\caption{Train (top) and test (bottom) averaged F1-scores for our method (blue) compared to other automatic feature selection techniques -- filter method (purple) and backward elimination (green) -- as well as no feature selection (dashed red).}
\label{fig:feature_selection_error}
\end{figure}

\pagebreak
\section{Conclusion and Future Work}
We have presented both theoretical and empirical evidence that local-membership queries are useful and beneficial.
In the theoretical setup we have shown that even 1-local queries are stronger than the vanilla PAC model in an arguably natural problem.
In the empirical setup we have demonstrated that by getting additional information from the users, significantly better results can be achieved.
Moreover, the data in the experiment was created using crowdsourcing, and by asking very simple questions. This shows that getting extra knowledge can be an easy task.

Today, the use of the MQ model in practice is almost non-existent. Even the more popular models of active learning, pool-based or stream-based, are fairly rare. E.g., in a recent survey of annotation projects for natural language processing tasks, only $20\%$ of the respondents stated they had ever decided to use active learning \citep{tomanek2009web}.
It seems that there is plenty of room for incorporating more profound human knowledge to the field of machine learning, especially since today this knowledge can be collected quite easily.

More concrete directions for future work include: developing, implementing and analyzing more algorithms that use (local) membership queries and investigating the strength and limitations of the general $O(1)$-local queries model. Some examples of open questions: Is the use of 2-local queries stronger than the use of 1-local queries on a natural environment? What are the limitations of a model that uses $O(1)$-local queries with comparison to the model of \citep{awasthi2012learning} that uses $\log(n)$-local queries?

\newpage
 
\bibliography{my_bib}
\bibliographystyle{apalike}
\end{document}

%% file: header.tex
\usepackage{graphicx,subfigure}
\usepackage{epsfig}
\usepackage{hyperref}
\usepackage{amsmath}
\usepackage{amssymb}
\usepackage{algorithm}
\usepackage{algorithmic}
\usepackage{url}
\usepackage{enumerate}
\usepackage{amsfonts}
\usepackage{boxedminipage}
\usepackage{xcolor}
 \usepackage{framed}  
\usepackage{rotating}
\usepackage{array}
\usepackage{multirow}
\usepackage{color}
\usepackage{tikz}
\usepackage{pgfplots}
\usepackage{mathabx}
\usepackage{tabularx,ragged2e,booktabs,caption}

{
\begin{center}
\begin{boxedminipage}{0.8\linewidth}
\begin{center}
\textbf{\texttt{#1}}
\end{center}
\rm
\begin{tabbing}
....\=...\=...\=...\=...\=  \+ \kill
} %
{\end{tabbing} 
\end{boxedminipage} \end{center} 
}

{
\vspace{0.01cm}
\begin{center}
\begin{boxedminipage}{0.8\linewidth}
\begin{center}
\textbf{\texttt{#1}}
\end{center}
} %
{
\end{boxedminipage} \end{center} \vspace{0.3cm}
}

\newtheorem{definition}{Definition}

\newtheorem{theorem}{Theorem}

\newtheorem{remark}{Remark}
\newtheorem{claim}{Claim}

\newcommand{\prob}{\mathbb{P}}
\newcommand{\bE}{\mathbb{E}}

\newcommand{\poly}{\mathrm{poly}}

\newcommand{\cY}{\mathcal{Y}}
\newcommand{\cX}{\mathcal{X}}
\newcommand{\cD}{\mathcal{D}}

\newcommand{\cB}{\mathcal{B}}

\newcommand{\cA}{\mathcal{A}}
\newcommand{\cH}{\mathcal{H}}

\newenvironment{proof}{\par\noindent{\bf Proof\ }}{\hfill\BlackBox\\[2mm]}
\newcommand{\BlackBox}{\rule{1.5ex}{1.5ex}}

\makeatletter
\def\moverlay{\mathpalette\mov@rlay}
\def\mov@rlay#1#2{\leavevmode\vtop{%
   \baselineskip\z@skip \lineskiplimit-\maxdimen
   \ialign{\hfil$\m@th#1##$\hfil\cr#2\crcr}}}
\newcommand{\charfusion}[3][\mathord]{
    #1{\ifx#1\mathop\vphantom{#2}\fi
        \mathpalette\mov@rlay{#2\cr#3}
      }
    \ifx#1\mathop\expandafter\displaylimits\fi}
\makeatother

\renewcommand{\eqref}[1]{Equation~(\ref{#1})}

\newcommand{\handout}[5]{
   \renewcommand{\thepage}{#1-\arabic{page}}
   \noindent
   \begin{center}
   \framebox{
      \vbox{
    \hbox to 5.78in { {\bf (67577) Introduction to Machine Learning}
         \hfill #2 }
       \vspace{4mm}
       \hbox to 5.78in { {\Large \hfill #5  \hfill} }
       \vspace{2mm}
       \hbox to 5.78in { {\it #3 \hfill #4} }
      }
   }
   \end{center}
   \vspace*{4mm}
}








%% file: thesis_archive.bbl
\begin{thebibliography}{}

\bibitem[Angluin, 1987]{angluin1987learning}
Angluin, D. (1987).
\newblock Learning regular sets from queries and counterexamples.
\newblock {\em Information and computation}, 75(2):87--106.

\bibitem[Angluin and Kharitonov, 1995]{angluin1995won}
Angluin, D. and Kharitonov, M. (1995).
\newblock When won′ t membership queries help?
\newblock {\em Journal of Computer and System Sciences}, 50(2):336--355.

\bibitem[Angluin et~al., 1997]{angluin1997malicious}
Angluin, D., Kri{\c{k}}is, M., Sloan, R.~H., and Tur{\'a}n, G. (1997).
\newblock Malicious omissions and errors in answers to membership queries.
\newblock {\em Machine Learning}, 28(2-3):211--255.

\bibitem[Angluin and Slonim, 1994]{angluin1994randomly}
Angluin, D. and Slonim, D.~K. (1994).
\newblock Randomly fallible teachers: Learning monotone dnf with an incomplete
  membership oracle.
\newblock {\em Machine Learning}, 14(1):7--26.

\bibitem[Awasthi et~al., 2012]{awasthi2012learning}
Awasthi, P., Feldman, V., and Kanade, V. (2012).
\newblock Learning using local membership queries.
\newblock {\em arXiv preprint arXiv:1211.0996}.

\bibitem[Balcan et~al., 2006]{balcan2006agnostic}
Balcan, M.-F., Beygelzimer, A., and Langford, J. (2006).
\newblock Agnostic active learning.
\newblock In {\em Proceedings of the 23rd international conference on Machine
  learning}, pages 65--72. ACM.

\bibitem[Barbosa and Feng, 2010]{barbosa2010robust}
Barbosa, L. and Feng, J. (2010).
\newblock Robust sentiment detection on twitter from biased and noisy data.
\newblock In {\em Proceedings of the 23rd International Conference on
  Computational Linguistics: Posters}, pages 36--44. Association for
  Computational Linguistics.

\bibitem[Baum, 1991]{baum1991neural}
Baum, E.~B. (1991).
\newblock Neural net algorithms that learn in polynomial time from examples and
  queries.
\newblock {\em Neural Networks, IEEE Transactions on}, 2(1):5--19.

\bibitem[Baum and Lang, 1992]{baum1992query}
Baum, E.~B. and Lang, K. (1992).
\newblock Query learning can work poorly when a human oracle is used.
\newblock In {\em International Joint Conference on Neural Networks}, volume~8.

\bibitem[Bisht et~al., 2008]{bisht2008learning}
Bisht, L., Bshouty, N.~H., and Khoury, L. (2008).
\newblock Learning with errors in answers to membership queries.
\newblock {\em Journal of Computer and System Sciences}, 74(1):2--15.

\bibitem[Blum et~al., 1995]{blum1995learning}
Blum, A., Chalasani, P., Goldman, S.~A., and Slonim, D.~K. (1995).
\newblock Learning with unreliable boundary queries.
\newblock In {\em Proceedings of the eighth annual conference on Computational
  learning theory}, pages 98--107. ACM.

\bibitem[Blum and Rudich, 1992]{blum1992fast}
Blum, A. and Rudich, S. (1992).
\newblock Fast learning of k-term dnf formulas with queries.
\newblock In {\em Proceedings of the twenty-fourth annual ACM symposium on
  Theory of computing}, pages 382--389. ACM.

\bibitem[Bshouty, 1995]{bshouty1995exact}
Bshouty, N.~H. (1995).
\newblock Exact learning boolean functions via the monotone theory.
\newblock {\em Information and Computation}, 123(1):146--153.

\bibitem[Daniely et~al., 2013]{daniely2013more}
Daniely, A., Linial, N., and Shalev-Shwartz, S. (2013).
\newblock More data speeds up training time in learning halfspaces over sparse
  vectors.
\newblock In {\em Advances in Neural Information Processing Systems}, pages
  145--153.

\bibitem[Daniely et~al., 2014]{daniely2014average}
Daniely, A., Linial, N., and Shalev-Shwartz, S. (2014).
\newblock From average case complexity to improper learning complexity.
\newblock In {\em Proceedings of the 46th Annual ACM Symposium on Theory of
  Computing}, pages 441--448. ACM.

\bibitem[Daniely and Shalev-Shwatz, 2014]{daniely2014complexity}
Daniely, A. and Shalev-Shwatz, S. (2014).
\newblock Complexity theoretic limitations on learning dnf's.
\newblock {\em arXiv preprint arXiv:1404.3378}.

\bibitem[Dasgupta, 2004]{dasgupta2004analysis}
Dasgupta, S. (2004).
\newblock Analysis of a greedy active learning strategy.
\newblock In {\em Advances in neural information processing systems}, pages
  337--344.

\bibitem[Davidov et~al., 2010]{davidov2010enhanced}
Davidov, D., Tsur, O., and Rappoport, A. (2010).
\newblock Enhanced sentiment learning using twitter hashtags and smileys.
\newblock In {\em Proceedings of the 23rd International Conference on
  Computational Linguistics: Posters}, pages 241--249. Association for
  Computational Linguistics.

\bibitem[Druck et~al., 2009]{druck2009active}
Druck, G., Settles, B., and McCallum, A. (2009).
\newblock Active learning by labeling features.
\newblock In {\em Proceedings of the 2009 Conference on Empirical Methods in
  Natural Language Processing: Volume 1-Volume 1}, pages 81--90. Association
  for Computational Linguistics.

\bibitem[Ehrenfeucht and Haussler, 1989]{ehrenfeucht1989learning}
Ehrenfeucht, A. and Haussler, D. (1989).
\newblock Learning decision trees from random examples.
\newblock {\em Information and Computation}, 82(3):231--246.

\bibitem[Feldman, 2009]{feldman2009power}
Feldman, V. (2009).
\newblock On the power of membership queries in agnostic learning.
\newblock {\em The Journal of Machine Learning Research}, 10:163--182.

\bibitem[Freund, 1995]{freund1995boosting}
Freund, Y. (1995).
\newblock Boosting a weak learning algorithm by majority.
\newblock {\em Information and computation}, 121(2):256--285.

\bibitem[Hu and Liu, 2004]{hu2004mining}
Hu, M. and Liu, B. (2004).
\newblock Mining and summarizing customer reviews.
\newblock In {\em Proceedings of the tenth ACM SIGKDD international conference
  on Knowledge discovery and data mining}, pages 168--177. ACM.

\bibitem[Jackson, 1994]{jackson1994efficient}
Jackson, J. (1994).
\newblock An efficient membership-query algorithm for learning dnf with respect
  to the uniform distribution.
\newblock In {\em Foundations of Computer Science, 1994 Proceedings., 35th
  Annual Symposium on}, pages 42--53. IEEE.

\bibitem[Kearns and Valiant, 1994]{kearns1994cryptographic}
Kearns, M. and Valiant, L. (1994).
\newblock Cryptographic limitations on learning boolean formulae and finite
  automata.
\newblock {\em Journal of the ACM (JACM)}, 41(1):67--95.

\bibitem[Kim and Hovy, 2004]{kim2004determining}
Kim, S.-M. and Hovy, E. (2004).
\newblock Determining the sentiment of opinions.
\newblock In {\em Proceedings of the 20th international conference on
  Computational Linguistics}, page 1367. Association for Computational
  Linguistics.

\bibitem[Klivans et~al., 2006]{klivans2006cryptographic}
Klivans, A.~R., Sherstov, A., et~al. (2006).
\newblock Cryptographic hardness for learning intersections of halfspaces.
\newblock In {\em Foundations of Computer Science, 2006. FOCS'06. 47th Annual
  IEEE Symposium on}, pages 553--562. IEEE.

\bibitem[Kouloumpis et~al., 2011]{kouloumpis2011twitter}
Kouloumpis, E., Wilson, T., and Moore, J. (2011).
\newblock Twitter sentiment analysis: The good the bad and the omg!
\newblock {\em Icwsm}, 11:538--541.

\bibitem[Kushilevitz and Mansour, 1993]{kushilevitz1993learning}
Kushilevitz, E. and Mansour, Y. (1993).
\newblock Learning decision trees using the fourier spectrum.
\newblock {\em SIAM Journal on Computing}, 22(6):1331--1348.

\bibitem[Nakov et~al., 2013]{nakov2013semeval}
Nakov, P., Kozareva, Z., Ritter, A., Rosenthal, S., Stoyanov, V., and Wilson,
  T. (2013).
\newblock Semeval-2013 task 2: Sentiment analysis in twitter.

\bibitem[Pak and Paroubek, 2010]{pak2010twitter}
Pak, A. and Paroubek, P. (2010).
\newblock Twitter as a corpus for sentiment analysis and opinion mining.
\newblock In {\em LREC}, volume~10, pages 1320--1326.

\bibitem[Pang and Lee, 2004]{pang2004sentimental}
Pang, B. and Lee, L. (2004).
\newblock A sentimental education: Sentiment analysis using subjectivity
  summarization based on minimum cuts.
\newblock In {\em Proceedings of the 42nd annual meeting on Association for
  Computational Linguistics}, page 271. Association for Computational
  Linguistics.

\bibitem[Pang and Lee, 2008]{pang2008opinion}
Pang, B. and Lee, L. (2008).
\newblock Opinion mining and sentiment analysis.
\newblock {\em Foundations and trends in information retrieval}, 2(1-2):1--135.

\bibitem[Raghavan and Allan, 2007]{raghavan2007interactive}
Raghavan, H. and Allan, J. (2007).
\newblock An interactive algorithm for asking and incorporating feature
  feedback into support vector machines.
\newblock In {\em Proceedings of the 30th annual international ACM SIGIR
  conference on Research and development in information retrieval}, pages
  79--86. ACM.

\bibitem[Raghavan et~al., 2005]{raghavan2005interactive}
Raghavan, H., Madani, O., and Jones, R. (2005).
\newblock Interactive feature selection.
\newblock In {\em IJCAI}, volume~5, pages 841--846.

\bibitem[Settles, 2010]{settles2010active}
Settles, B. (2010).
\newblock Active learning literature survey.
\newblock {\em University of Wisconsin, Madison}, 52(55-66):11.

\bibitem[Settles, 2011]{settles2011closing}
Settles, B. (2011).
\newblock Closing the loop: Fast, interactive semi-supervised annotation with
  queries on features and instances.
\newblock In {\em Proceedings of the Conference on Empirical Methods in Natural
  Language Processing}, pages 1467--1478. Association for Computational
  Linguistics.

\bibitem[Sloan and Tur{\'a}n, 1994]{sloan1994learning}
Sloan, R.~H. and Tur{\'a}n, G. (1994).
\newblock Learning with queries but incomplete information.
\newblock In {\em Proceedings of the seventh annual conference on Computational
  learning theory}, pages 237--245. ACM.

\bibitem[Tomanek and Olsson, 2009]{tomanek2009web}
Tomanek, K. and Olsson, F. (2009).
\newblock A web survey on the use of active learning to support annotation of
  text data.
\newblock In {\em Proceedings of the NAACL HLT 2009 Workshop on Active Learning
  for Natural Language Processing}, pages 45--48. Association for Computational
  Linguistics.

\bibitem[Valiant, 1984]{valiant1984theory}
Valiant, L.~G. (1984).
\newblock A theory of the learnable.
\newblock {\em Communications of the ACM}, 27(11):1134--1142.

\end{thebibliography}
